\documentclass{article}

\usepackage{arxiv}

\usepackage[utf8]{inputenc} 
\usepackage[T1]{fontenc}    
\usepackage{hyperref}       
\usepackage{url}            
\usepackage{booktabs}       
\usepackage{amsfonts}       
\usepackage{nicefrac}       
\usepackage{microtype}      
\usepackage{lipsum}
\usepackage{graphicx}  
\usepackage{amsmath}
\usepackage{amsthm,appendix}
\usepackage{caption}

\title{Near-Optimal Path Planning for Complex Robotic Inspection Tasks}

\newcommand{\afil}{Faculty of Applied Engineering\\ University of Antwerp\\Antwerp, Belgium\\}

\author{
  Boris Bogaerts\\
  \afil
  \texttt{boris.bogaerts@uantwerpen.be} \\
   \And
   Seppe Sels\\
   \afil
   \texttt{seppe.sels@uantwerpen.be}
   \AND
  Steve Vanlanduit \\
  \afil
  \texttt{steve.vanlanduit@uantwerpen.be} \\
  \And
  Rudi Penne \\
  \afil
  \texttt{rudi.penne@uantwerpen.be} \\
}

\usepackage{algorithm}
\usepackage{algpseudocode}
\usepackage{mathabx}
\usepackage{amssymb}
\usepackage{hyperref}

\newtheorem{proposition}{Proposition}

\usepackage{amsmath}               
  {
      \theoremstyle{plain}
      \newtheorem{assumption}{Assumption}
  }
\newcommand{\optionalBreak}{}

\begin{document}
\maketitle

\begin{abstract}
In this paper we consider the problem of generating inspection paths for robots. These paths should allow an attached measurement device to perform high quality measurements. We formally show that generating robot paths, while maximizing the inspection quality, naturally corresponds to the submodular orienteering problem. Traditional methods that are able to generate solutions with mathematical guarantees do not scale to real world problems. In this work we propose a method that is able to generate near-optimal solutions for real world complex problems. We experimentally test this method in a wide variety of inspection problems and show that it nearly always outperforms traditional methods. We furthermore show that the near-optimality of our approach makes it more robust to changing the inspection problem, and is thus more general.
\end{abstract}


\begin{figure}[H]
    \centering
     \href{https://youtu.be/Fg-ulGRyw2w}{
    \includegraphics[width=0.85\textwidth]{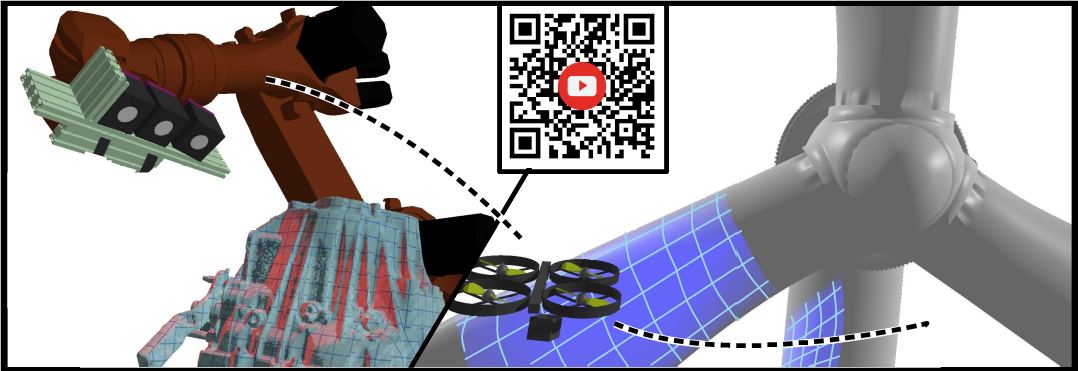}}
    \caption{A 360vr video experience that explains and visualizes this work is available online (https://youtu.be/Fg-ulGRyw2w). This video can be watched on a regular computer, or on a smartphone, but the optimal experience requires a virtual reality headset. Click this figure or scan the QR code to get redirected.}
\end{figure}

\section{Introduction}

In robotic inspection, a measurement device (e.g. a camera) is attached to a robot. The goal of inspection planning is then to find an efficient robot trajectory that allows the attached measurement device to perform high quality, complete measurements. The abstract structure of this problem is shared by many real-world inspection tasks featuring different robots and measurement devices. Examples include but are not limited to:
\begin{itemize}
    \item Inspection of wind turbines with a drone using optical cameras \citep{bircher2016three}.
    \item Inspection of solar power plants with a drone using thermal cameras \citep{quater2014light}.
    \item Inspection of aircraft structures with a robotic manipulator and active thermography \citep{peeters2019optimized}.
    \item Inspection of ships hulls using Autonomous Underwater Vehicles \citep{englot2012sampling}.
\end{itemize}

The wealth of practical applications for automated inspection algorithms generated academic interest, making it a well-studied problem. The literature on inspection planning generated two major algorithmic approaches. The first class of algorithms formally model the inspection planning problem as an optimization problem, and algorithms are constructed to generate solutions that score well on this optimization problem. A major issue with these algorithms is that the abstract problem of inspection planning is a very challenging combination of NP-hard problems as we will discuss later. This results in algorithms that can typically only solve small scale problems (i.e. toy problems). The other class of algorithms originated from a more pragmatic view and try to solve more realistically complex problems. These algorithms typically focus on concrete specific problems, but often fail to generalize to other related problems (e.g. different robot, different measurement devices). Another disadvantage is that there is no guarantee on the quality of the solutions that are returned by the algorithm. 

In this work, we will construct an automatic inspection planning algorithm that aims to connect both classes of algorithms. This algorithm generates near-optimal solutions, while at the same time it is able to solve realistically large scale complex inspection problems. Another advantage of our algorithm is that it is general enough to solve a large variety of inspection planning problems. The proposed algorithm also generates a bound on the maximally achievable quality which is valuable in the comparison with other algorithms.

\section{Related work}
Early works that focus on practical aspects of the inspection planning problem are concerned with the next-best-view planning problem \citep{pito1999solution}. These articles typically focus on the construction of functions that model the quality of a measurement \citep{pito1999solution, scott2009model, trummer2010online}. However to find a solution to the inspection planning problem, the proposed methods use a two-step procedure that separates the viewpoint selection and path planning problem. As the name suggests, the planning horizon is also limited to one step ahead planning. There are no mathematical guarantees on the optimality of the solutions generated by these procedures.

Early work on the coverage path planning defines the planning task as finding a robot trajectory that covers the entire space of interest \citep{choset1998coverage}. The inspection requirement is simplified to the point that a collection of points must be visited by the robot, and thus the focus is only on the robot trajectory \citep{galceran2013survey}. 

In contrast \cite{bircher2016three} propose a sampling-based procedure that considers both inspectability of a path and its length. However, their procedure has no formal optimality guarantees and is specifically designed for UAVs. \cite{roberts2017submodular} formally model the inspection problem as an optimization problem, and use a branch-and-bound solver on a relaxation of this problem. The branch-and-bound solution method, however, limits the size of the problem instances that can be solved. The relaxation that is performed also breaks all optimality guarantees.  

\cite{englot2012sampling} propose a sampling-based procedure that is guaranteed to converge to the optimal solution. The main disadvantage of this approach is that the rate of convergence is unknown and most likely slow, especially considering that the inspection planning problem is typically a large NP-hard problem. Another disadvantage is that the inspection quality is treated as a binary variable. In reality, however, the inspection quality depends on the measurement conditions. \cite{papadopoulos2013asymptotically} propose an algorithm with similar properties that extends to robots with more challenging constraints, such as a robotic manipulator. However, this approach is only of theoretical interest as it does not scale well to large problems.

\cite{singh2009nonmyopic} model the inspection task as maximizing the mutual information of a Gaussian Process, and present a near-optimal algorithm to find a walk in a graph with the aim to maximize the inspection quality. However, the proofs in this work rely on a very specific instance of submodularity specific to the mutual information of a Gaussian Process (i.e. $(r,\gamma)$-local submodularity). Furthermore, this algorithm has been shown to scale poorly to larger problems \citep{roberts2017submodular}.

A closely related problem, proposed by \cite{yu2014correlated} is the correlated orienteering problem (COP). This problem is used to model a persistent monitoring task using drones and can be solved by mixed integer quadratic programming. The disadvantage, however, is that the structure of the COP formulation is not rich enough to model realistic complexities typically related to functions modelling inspection quality. Another disadvantage is that this method was only shown to work with small problem instances (i.e. a workspace grid of 7$\times$7 nodes in 2D). 

In a previous paper, we proposed a local inspection path optimization approach \citep{bogaerts2018gradient}. This technique however relies on a good initial trajectory and does not provide any optimality guarantees. Local inspection path optimization approaches without guarantees can benefit from an initialization with formal quality guarantees.

\section{Abstract problem structure}
\label{sec::Abstract}
As mentioned earlier, the goal of robotic inspection planning is to find an efficient robot trajectory that allows an attached measurement device to perform high quality, complete measurements. It is a problem with at its core two conflicting interests. The first interest is to maximize the inspection quality while the second interest is to keep the robot trajectory efficient. Separately, both maximizing the inspection quality and finding a minimal trajectory are NP-hard problems. What counts as a cost differs per application but can for example be, travelling distance, inspection time, etc. We will cover both aspects, and how they are connected in the following sections.

In this work, we will assume that information about the object that is being inspected, the robot system, and the characteristics of the measurement device are known beforehand. Information about the measurement object is in the form of a CAD model. We also assume that a digital twin of the robot system in an accurate representation of the environment is available to make sure that the final trajectory is kinematically reachable. We finally assume that the characteristics of the measurement device, that quantify the expected measurement quality as a function of the measurement conditions are available. 

An instance of the inspection planning problem is in reality continuous. To make the problem tractable, we will discretize the problem to a discrete problem. Different aspects of the discretization are shown in Fig. \ref{fig:sheetSheet}. Each of these aspects will be discussed in separate subsections. The cardinalities provided in Fig.\ref{fig:sheetSheet} are used throughout this work, without repeating their meaning.
\begin{figure*}
    \centering
    \includegraphics[width=\textwidth]{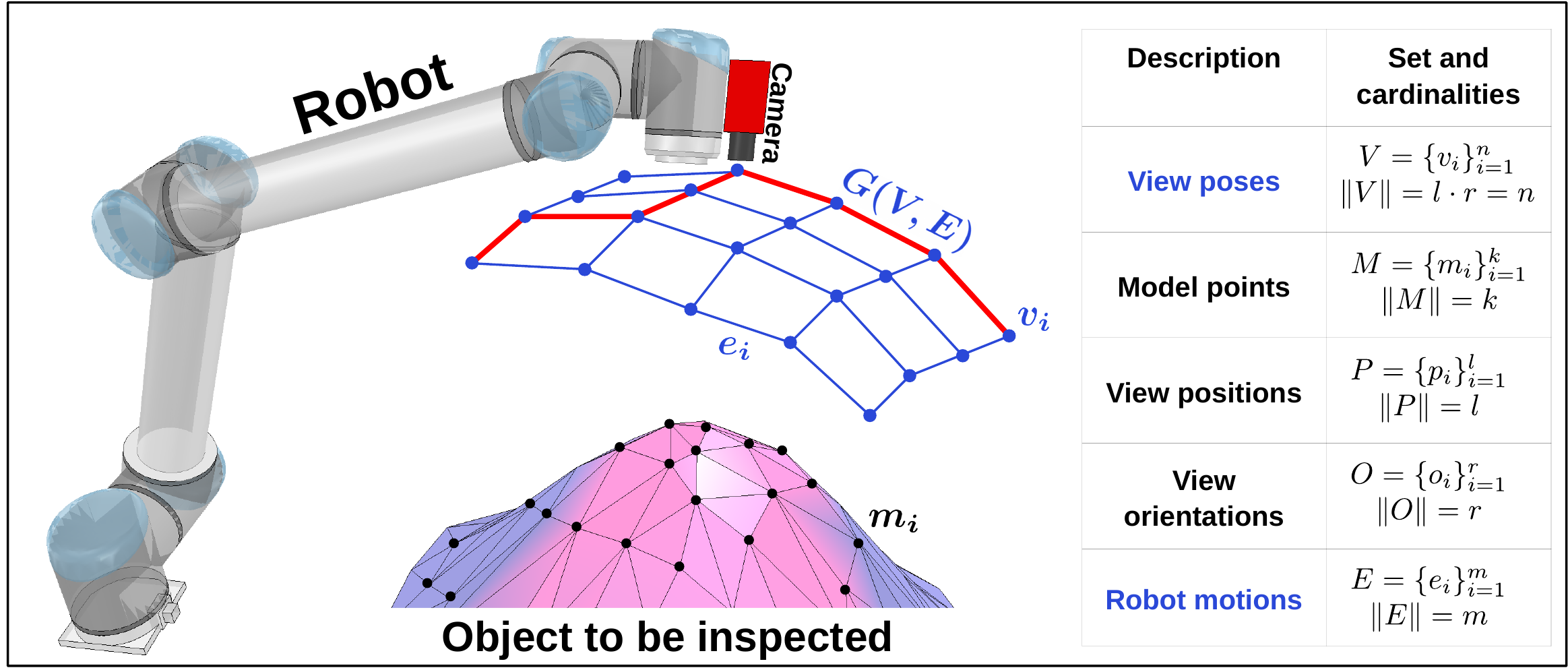}
    \caption{This figure connects the abstract structure of the inspection planning problem and its notation to their physical interpretation.}
    \label{fig:sheetSheet}
\end{figure*}

\subsection{Inspection quality}

In this section, we will show that functions that model inspection quality are naturally connected to submodular functions. This connection allows us to formally characterize the inspection planning problem, and study the effectiveness of algorithms. Specific instances of functions modelling inspection quality have already been shown to be submodular \citep{roberts2017submodular}, but a recent work on camera network performance models highlights the generality of this connection \citep{Bogaerts2019Vr}. Submodular set functions are set functions that are characterized by a \emph{diminishing returns property}. More formally:\\
Let $\Omega$ be a finite set, and let $f : 2^\Omega \longrightarrow \mathbb{R}$ be a real-valued function of subsets of $\Omega$, then $f$ is called
{\em submodular\/} if for each $X\subseteq Y \subseteq \Omega$ and for each $x\in \Omega\setminus Y$ it holds that 
\begin{equation}
    f(X \cup \{ x\}) - f(X) \geq f(Y \cup \{x\}) - f(Y)
    \label{eqn:submod}
\end{equation}

This diminishing returns property models the fact that inspecting some part of the object becomes less interesting after it has already been inspected, in the context of inspection planning. To relate functions that quantify inspection performance to submodular functions we start by discretizing the object that we wish to inspect $M=\{m_1,\ldots,m_k\}$ (model) into a finite collection of points. Furthermore, we assume the availability of a finite sample $V$ of relevant view poses for the inspecting camera, $V=\{ v_1,\ldots, v_n\}$. Each view pose $v\in V$ is a pair $(p,o)$ consisting of a 3D-point $p$ (position) and an orientation $o$ (Fig.\ref{fig:sheetSheet}). 
In Section~\ref{sec::tspCost} we explain how to obtain $V$ by a separate sampling of view positions $P$ and orientations $O$. For each model point $m\in M$ we can consider the subset view poses, $R_m=\{v\in V: m \mbox{ is visible from } v\}$. While $R_m$ represents a binary inspection quality of view poses in a given model point $m\in M$, we will also introduce more general quality rate functions $q_m\; : V\longrightarrow\mathbb{R}^+$, 
representing a more nuanced quality of the view poses in a given inspected model point (Section ~\ref{sec::obtainQual}). In this setting, we agree that high values $q_m(v)$ correspond to high inspection quality by $v\in V$. Furthermore, we assume that $q_m(v)=0$ if $v\not\in R_m$. 

In general, the total inspection quality for a given model point $m$ is the accumulation of the inspection rates $q_m(v)$ by the fusion of the sensor data captured at several view poses, visited during the inspection trajectory. Locally at each model point $m\in M$, this sensor accumulation is formally ruled by a function $G_m: 2^V\times V\longrightarrow \mathbb{R^+}$. More precisely, for each subset $X$ of view poses and for each individual view pose $v\in V$, the nonnegative value $G_m(X,v)$ represents the added inspection quality at $m$ by adding $v$, relative to previous inspections by $X$. Such a {\em marginal quality function\/} $G_m$ is associated with a given family of quality rate functions ${\cal F}=\{q_m\; |\; m\in M\}$, and we call $G_m$ an {\em unordered and decreasing\/} marginal quality function if the following conditions are satisfied:
\begin{enumerate}
\item $G_m(\emptyset,v)=q_v(m)$ for each $v\in V$ (initialized by ${\cal F}$)
\item $G(X,v)=0$ if $v\in X$ (saturation, i.e. measuring a model point under exactly the same conditions twice, is not beneficial.)
\item If $X=\{v_1,\ldots,v_s\}\subseteq V$ than the sum
$$G_m(\emptyset,v_1)+G_m(\{v_1\},v_2)+\cdots G_m(X\setminus\{v_s\},v_s)$$
is independent of the chosen order in $X$. (unordered)
\item Extending the set of view poses $X\subseteq Y\subseteq V$, decreases the relative benefits: $G_m(Y,v)\le G_m(X,v)$. (decreasing)
\end{enumerate} 

Now we can state the inspection objective as a submodular function on subsets $X\subseteq V$ of view poses. More precisely, for a fixed $m\in M$, we formalize the inspection quality $f_m(X)$ due to viewpoints $X\subseteq V$ recursively by 
\begin{eqnarray}
f_m(\emptyset) &=& 0\\
f_m(X) &=& f_m(X\setminus\{ v\}) + G_m(X\setminus\{ v\},v)\;\;\optionalBreak\mbox{if }X\neq\emptyset\;\mbox{ and }v\in X
\label{eq:process}
\end{eqnarray}
Notice that $f_m(\{ v\}) = q_m(v)$.

{\bf Examples.}
\begin{enumerate}
\item For a given object point $m$ and an arbitrary family of quality rate functions ${\cal F}=\{q_m\; |\; m\in M\}$, we can define
\begin{eqnarray*}
G_m(X,v) &=& 0, \mbox{ if } q_m(v)\le\max\{q_m(w)\; |\; w\in X\},\\
G_m(X,v) &=& q_m(v) - \max\{q_m(w)\; |\; w\in X\},\optionalBreak \mbox{ otherwise.}
\end{eqnarray*}
that can be seen to be an unordered and decreasing marginal quality function, yielding the submodular inspection objective:
$$f_m(X) = \max\{q_m(w)\; |\; w\in X\}$$ 
\item If we use a binary quality rate, with $q_m(v)=1 \iff v\in R_m$, and if we define a binary $G_m$ by
$$G_m(X,v)=1 \iff v\not\in X\mbox{ , }v\in R_m{ \textrm{ and } }f_m(X\setminus\{v\})=0$$
then we just obtain $f_m(X) = 1 \mbox{ if }\#(X\cap R_m)>0$. Note that this exact objective is used frequently \citep{englot2012sampling,scott2009model,papadopoulos2013asymptotically,bogaerts2018gradient, Bogaerts2019Vr}.
\end{enumerate}

\begin{proposition}
Let $m$ be a fixed model point and let $q_m\; : V\longrightarrow\mathbb{R}^+$ be a given quality rate in $m$, supporting a marginal
quality function $G_m$. Then the associated
quality function $f_m\; :\; 2^V \longrightarrow\mathbb{R}^+$ is a well-defined, monotone increasing, submodular function.
\end{proposition} 

\begin{proof}
First at all, $f_m$ is well-defined by the recursion of Eqn.~\ref{eq:process}, because $G_m$ is unordered.\\
Next, $f_m$ is monotone increasing because $G_m$ delivers positive values.\\
Finally, $f_m$ is submodular because $G_m$ is supposed to be decreasing.
\end{proof}

Of course, our ultimate goal is to maximize the inspection quality for the global object, represented by the sample $M$.
This motivates us to define 
$$f(X) = \sum_{m\in M} f_m(X)$$
which is monotone increasing and submodular as  a sum of monotone submodular functions.

\subsection{Travelling costs}
\label{sec::travellingCost}
In order to compute a travelling cost for subsets of view poses $X\subseteq V$, we start by representing the space of all robot motions as a finite connected graph $G(V,E)$ (see Fig.\ref{fig:sheetSheet}). For now, we will assume the existence of edges in this graph, in Section \ref{sec::tspCost} we will provide more details. The edges of this graph $E=\{e_1,...,e_m\}$ represent robot motions between view poses. Traversing an edge of this graph will result in an associated cost $c(e)$. For each collection of view poses $X\subseteq V$, $c(X)$ collects the costs of all edges on the  walk of minimum cost through $G$ that passes through all $v\in X$. We model the cost of a subset of view poses as:
\begin{equation}
    C(X) = c(X) + \alpha |X|
    \label{eq:alpha}
\end{equation}
Here, $\alpha$ models the cost associated with performing a measurement in each view pose. Note that in order to evaluate $c(X)$ we need to solve the well-known travelling salesman problem. In the remainder of this section, we will highlight some important aspects of the travelling salesman problem which will be used later in this work.

While the travelling salesman problem (TSP) is NP-hard, it is possible to quickly solve this problem for graphs of up to 1000 nodes using the Dantzig-Fulkerson formulation \citep{applegate2003implementing}. This formulation transforms the TSP problem in an integer linear programming problem that can be solved to optimality with branch-and-bound solvers. In our implementation, we make use of the Gurobi solver \citep{gurobi} to solve the integer linear programming problem.

Obtaining an exact solution to the TSP problem does however come at a computational cost. Our final algorithm requires the computation of many TSP problems (i.e. $O(n)$, $n$ is the number of view poses) making it too expensive for realistic problems. Therefore, we make use of an algorithm that generates approximate solutions to the TSP problem. Approximate solutions are solutions that are at most a constant factor larger than the exact solution. In this work, we make use of the nearest-neighbour algorithm that generates a $O(log(n))$-approximate solution to the TSP problem \citep{rosenkrantz1977analysis}. We will refer to the cost obtained by the nearest neighbour algorithm by $\widecheck{C}$ (the accent on C points downwards to indicate that the exact solution is smaller).

While the nearest neighbour algorithm generates an approximate solution, which can serve as an upper bound for the exact value, we also make use of a lower bound. This lower bound is known as the Held-Karp bound (HK). The HK bound is the solution to the relaxation of the linear programming formulation of the travelling salesman problem \citep{reinelt1994traveling}. \cite{held1970traveling} proposed an iterative approach to obtain this bound quickly. While the HK bound is only guaranteed to give a solution that is never less than 2/3 of the minimum cost, it in reality performs much better. The HK bound was shown to generate solutions to real problems with a gap of less than 1\% to the exact solution \citep{valenzuela1997estimating}. We will refer to the HK bound by $\widehat{C}$. \cite{valenzuela1997estimating} show that the HK bound can be reliably estimated with an algorithm with a time complexity of $O(n log(n))$. This algorithm works by generating a sequence of minimum 1-trees that converge to the linear programming relaxation of the TSP problem. 

\subsection{The submodular orienteering problem and the Generalized Cost-Benefit Algorithm}
The problem combining submodular function maximization subject to an upper bound on the travelling cost constraint is known as the submodular orienteering problem \citep{chekuri2005recursive}. This problem is formally given by:
\begin{equation}
    X^*=argmax_{X\subset V}\{f(X) | C(X)\leq B \}
    \label{eq:theproblem}
\end{equation}
Here, $B$ is the maximum allowed travelling budget. In the context of inspection planning, this problem aims to maximize the inspection quality with a constraint on the maximum inspection cost. \cite{zhang2016submodular} propose the generalized cost-benefit algorithm (GCB) which is a polynomial time algorithm with provable approximation guarantees. The fact that the algorithm is polynomial is interesting in the context of inspection planning because typical problem instances tend to be very large. This algorithm guarantees a $0.5(1-e^{-1})$ ($\approx 0.32$) solution to the submodular orienteering problem relative to a value $f(\tilde{X})$, where $\tilde{X}$ is the optimal solution of $max\{f(X)|\widetilde{C}(X)\leq k B/\psi(n)\}$. Here, $k$ is a constant larger than, but close to 1 and $\widetilde{C}$ is a submodular function mimicking the behavior of the true cost function \citep{zhang2016submodular}. $\psi(n)$ is the approximation guarantee of the travelling salesman algorithm that is used in the GCB algorithm. This bound on the performance of the algorithm is however overly pessimistic, as we will show in our experiments. We discuss a tighter problem specific bound in Section. \ref{sec:tight}

The GCB algorithm is provided in Algorithm \ref{alg:algorithm1}. Also, notice that $O(n)$ instances of the TSP problem need to be solved in this algorithm which is the reason for resorting to approximate solutions. 

\begin{algorithm} 
  \begin{algorithmic}[1]
      \State $ X \gets \emptyset$
      \While{$C(X)\leq B$ }
        \ForAll {$ \:x\in V$}
            \State $\Delta^x_f = f(X\cup \{x\}) - f(X)$
            \State $\Delta^x_{\widecheck{C}} = \widecheck{C}(X\cup \{x\}) - \widecheck{C}(X)$
        \EndFor
        \State $x^* = argmax(\Delta^x_f/\Delta^x_{\widecheck{C}})$
        \State $X = X \cup \{x^*\}$
    \EndWhile
    \State \Return $X\setminus \{x^*\}$
  \end{algorithmic} 
  \caption{Generalized Cost-Benefit Algorithm (GCB)}
  \label{alg:algorithm1}
\end{algorithm}

\subsection{Improving the solution of the GCB algorithm}
\label{sec:improvement}
In this section, we will propose an algorithm that is performed on the solution of the GCB algorithm. This step is a variation of a step that is often implicitly performed in the submodular function community but rarely discussed explicitly. This step tries to replace elements of the final set if new elements have greater marginal function values than old elements. This step often has a significant impact on the final quality of the solution. Our variation of this step is provided in Algorithm \ref{alg:algorithm2}. 

   \begin{algorithm}[H]
  \begin{algorithmic}[1]
      \State $X \gets GCB$
      \State $\delta f \gets \infty$
      \While{$\delta f > 0$}
      \State $f_0\gets f(X)$
      \ForAll{$ \: x\in X$} 
          \State $X^- \gets X\setminus \{x\}$
          \State $X_t \gets V\setminus X^-$
        \ForAll{$ \:v\in X_t$}
            \State $\Delta^v_f \gets f(X^-\cup \{v\}) - f(X^-)$
            \State $\widecheck{C}^v \gets \widecheck{C}(X^- \cup \{v\})$
        \EndFor
        \State $v^* \gets lazy_{argmax}\{\Delta^v_f | C(X^-\cup \{v\})\leq B\}$
        \State $X\gets X \cup \{v^*\}$
    \EndFor
    \State $\delta f \gets f(X)-f_0$
    \EndWhile
    \State \Return $X$
  \end{algorithmic} 
  \caption{GCB$^+$}
  \label{alg:algorithm2}
 \end{algorithm}

\begin{algorithm}[H]
  \begin{algorithmic}[1]
    \State $C^* \gets B$
    \While{$C^*>B$}
        \State $v^* \gets argmax(\Delta^v_f | \widecheck{C}(X^-\cup \{v\})\leq B\})$
        \State $\Delta^{v^*}_f \gets 0$
        \State $C^* \gets C(X^-\cup \{v^*\})$
    \EndWhile
    \State \Return $v^*$
  \end{algorithmic} 
  \caption{$lazy_{argmax}$}
  \label{alg:algorithm3}
 \end{algorithm}

Our variation of this extra step is designed explicitly for the submodular orienteering problem. The usual step would replace elements with new elements if their cost-benefit ratio is higher. In our algorithm, the cost-benefit ratio will not be considered, but the focus will be on just the benefit. We however only consider elements that when added, result in a solution that has a cost lower than the budget. The benefits of this heuristic are twofold. The first benefit is that the Nearest neighbour algorithm is only precise over large instances of the TSP problem (i.e. $O(log(n))$), making the error in the estimation of the cost-benefit ratio relatively large. This error will be large because the gap over which decisions to include elements need to be made ($C(X^-\cup v) - C(X^-)$) can be relatively small. In our heuristic, we do not need to estimate the cost-benefit ratio. The other reason is that elements of $V$ with the largest marginal return could have been ignored during the optimization phase, because of a large associated cost relative to a partial solution. In our extension, these elements will be reconsidered with respect to a set that is close to the final trajectory.

In this algorithm, we also compute $\widecheck{C}^v$ (the HK bound), seemingly for no particular reason. Its use is however hidden in the $lazy_{argmax}$ step where $C^v$ (the exact cost) is computed (see Algorithm \ref{alg:algorithm3}). Since the HK bound is a lower bound, we can safely ignore all values with an HK bound that is greater than $B$. The practical tightness of the HK bound ensures that only a minimal number of more expensive $C^v$ problems need to be solved. 

\subsection{Obtaining a tight reference measure $OPT$}
\label{sec:tight}
As mentioned earlier, the guarantee of the GCB algorithm is not tight. In this section, we will provide details on how to compute a much tighter problem specific bound. We also suggest to use this bound to compare the quality of different algorithms. Since this bound can be computed for large scale realistic problems, it can also be used to quantify the performance of algorithms that do not have formal optimality guarantees. 

The GCB algorithm produces a partial solution at each time step $\{X_1,X_2,...,X_t, X_{t+1}\}$ until the budget $B$ is violated at step $t+1$. The value $OPT$ is then given by:
\begin{equation}
    OPT = \frac{f(X_{t+1})}{1-\prod_{k=2}^{t+1}\left ( 1 - \frac{c(X_k)-c(X_{k-1})}{B}\right )}
\end{equation}

\begin{proposition}
Let $X_{t}$ be the subset generated by the GCB algorithm, and $X_{t+1}$ be the first subset for which $C(X_{t+1})>B$ then OPT is a tighter problem specific bound than $\frac{f(X_t)}{\frac{1}{2}(1-e^{-1})}$ i.e.:
$$ \frac{f(X_t)}{\frac{1}{2}(1-e^{-1})}\geq \frac{f(X_{t+1})}{(1-e^{-1})}\geq OPT \geq f(\tilde{X}) \geq f(X_t)$$ 

Where $\tilde{X}$ is the optimal solution of $max\{f(X)|\widetilde{C}(X)\leq k B/\psi(n)$, and $k$ is a constant close to one, as defined by \cite{zhang2016submodular}.
\end{proposition}

\begin{proof}
This statement follows directly from the proof of Theorem 1 of \cite{zhang2016submodular}.
\end{proof}

The factor $(1-e^{-1})$ is incidentally also the performance of the greedy algorithm on the monotone submodular function maximization under cardinality constraints \citep{Krause12submodularfunction}. This bound is known to be tight. This suggests that $OPT$ is also tight since \cite{zhang2016submodular} use the same strategy that is used to prove the boundedness of the greedy algorithm, in their proofs. Tightness implies that it is not possible to find a tighter bound without making more assumptions on $f$. The most popular structure in submodular functions, namely curvature, does for example not apply to submodular functions modelling inspection performance.

It is important to stress that the guarantees from which $OPT$ is derived do not include the actual cost $C$. These guarantees depend on a placeholder submodular function $\widetilde{C}$ that captures important characteristics of the true cost function $C$ \citep{herer1999submodularity}.
\section{Practical implementation}

In Section \ref{sec::Abstract}, we discussed the abstract structure of the inspection planning problem and presented an algorithm to solve it. In this section, we will discuss how we can use this approach to solve a practical inspection planning problem. Important in our practical implementation is that we want to keep the theoretical guarantees from Section \ref{sec::Abstract} intact. On the other hand, we also require that the theoretical assumptions make sense in the context of practical problems. 

Another important aspect is that we strive for generality. This means that the implementation should be able to solve planning problems for drones and robotic manipulators. Furthermore we cannot make particular assumptions on the quality function $f$ or cost function $C$ which need to tailor to a wide variety of specific problems. 

\subsection{Obtaining TSP costs}
\label{sec::tspCost}
An important aspect of robotic inspections is the robotic system that generates the motion of the measurement device. So it is no surprise that a practical implementation is built around dealing with this system efficiently. In this work, we assume that we can query the validity of robot states in a robot simulation environment. This validity entails that the robot can move the measurement device to a certain position and orientation without colliding with the environment. In our implementation, we use the robot simulation environment V-REP \citep{vrep}. A first step in encoding the robot limitations efficiently is to construct a view pose space discretization.  

To guarantee a general approach, we will perform the discretization of the view poses in task-space. The alternative, discretizing configuration-space would result in a discretization with low quality in the case of high dimensional state spaces. We will start by defining a specific discretization approach, and proceed by pointing out the advantages of this discretization.

The discrete position set $P$ is generated by constructing a regularly spaced grid in the inflated convex hull of the object that needs to be inspected. All the points that cannot be reached by the robot in any orientation are removed from the set $P$. Orientations are generated randomly and uniformly \citep{shoemake1992uniform}, yielding the finite set $O$..
The path planning is than executed on a graph $G=(V,E)$, where each vertex $v$ in $V$ corresponds to {\em fiber\/} $p^O=\{p\}\times O$ that represents all the poses in one specific position $p\in P$, considering each discrete orientation. Two such vertices $p_i^O$ and $p_j^O$ are connected by an edge $e_{i,j}\in E$ if and only the positions $t_i$ and $t_j$ are connected by a grid edge or a grid diagonal. Note that the reachability of positions at specific orientations was not evaluated during the construction of graph $G$. We will however check the reachability of positions at specific orientations in a lazy fashion at a later stage. This choice to only check the reachability of positions (rather than poses) results in a reduced number of cheaper reachability queries ($O(l)$ vs $O(n)$, see Fig.\ref{fig:sheetSheet}) to the robot simulator. The cost between these two view poses $(p_i,o_i)$ and $(p_j,o_j)$ is modelled as an edge cost given by:
\begin{equation}
c(e_{i,j})=(1-\beta)d_t(p_i,p_j)+\beta d_o(o_i,o_j)
\label{eq:dist}
\end{equation}

Here $d_t$ and $d_o$ are metric distance functions (i.e. the triangle inequality is satisfied) between positions and rotations, and $\beta$ is a weighting parameter. The main consequence of this modelling choice is the fact that the cost between orientations and positions are separated. Practical cost functions also motivate this choice. For inspections with drones, this type of cost function can model the total travelling cost, which must be bounded due to battery life. For inspections using robotic manipulators, this cost is proportional to the distance travelled by the measurement device, of which the derivative (i.e. speed) is limited. An example where this cost function fails is when energy consumed by a robot manipulator is considered as a cost. These types of costs are however very unlikely to play a role in robotic inspection problems, because the speed of the inspection is the main concern in the economics of robotic inspections.

Now that edge costs $c(e)$ are defined, we can run all TSP related algorithms on graph $G$. The discretization and distance matrices are pre-computed for both positions and orientations. The computation of the distance matrix in the case of the positions requires the computation of all shortest paths between every combination of positions. With Johnson’s algorithm \citep{Johnson1977} this step has a time complexity of $O(l^2\log(l) + l^2)$. Note that we made use of the fact that the number of edges is $O(l)$ in our proposed graph. Also
note that the choice to separate positions from orientations resulted in a highly reduced computational cost. In our
experiments, we will show that in large problem instances, this computational cost is acceptable. and can be used by TSP queries which are performed interactively during the optimization phase. This means that a cost $c(e)$ is only computed when needed. For every subset of view poses in a TSP query, a sub-matrix can quickly be extracted as the sum of a subset of the two distance matrices.

\subsection{Obtaining measurement quality}
\label{sec::obtainQual}
While the calculation of the TSP costs already required the discretization of the view pose space, calculating the measurement quality also requires the discretization of the object that needs to be inspected. Many methods exist that accomplish this, and the right choice depends on the application. In all our experiments we will randomly and uniformly sample the input mesh. The surface normal associated with each point in the discretization is inherited from the triangle of which the point was sampled. This normal can later be used in determining the expected measurement quality from different viewpoints. Determining the measurement quality for each view pose-surface point combination is achieved in three steps.

In the first step, the visibility is calculated using a highly efficient ray-tracer \citep{Embree}. Note that because the orientations and positions of the viewpoints are independent, we only need to perform $l\times k$ visibility computations. The second step evaluates for each visible view pose-surface point combination, for which orientations of the measurement device, the surface point is in the view frustum of the sensor located in the view pose. Finally, the expected measurement quality is calculated for the remaining view pose-surface point combinations. These quality values are pre-computed and stored in a sparse matrix. How this quality is computed is dependent on the physics of the measurement technique that is being used, and thus highly dependent on the measurement technique. In our implementation, we make use of a GPU to quickly evaluate the marginal benefit of adding a viewpoint in parallel. We do this mainly because the time complexity of this evaluation is $O(nk)$. In our largest experiment for example $n=351.10^3$ and $k=30.10^3$ (ignoring the sparsity). 

\subsection{Optimization}
The optimization phase uses the aforementioned pre-computed quality values and distance matrices. Given the pre-computed data it is straightforward to implement Algorithms \ref{alg:algorithm1} and \ref{alg:algorithm2}. The final measurement path returned by the algorithms is however not guaranteed to be executable by the robot. Note that the following limitations of the robot were neglected during the construction of the problem:
\begin{enumerate}
    \item Reachability of specific orientations at positions is not guaranteed.
    \item It is not guaranteed that the robot can execute each path that is considered by the optimizer.
\end{enumerate}
The reachability of specific orientations will be checked lazily just before adding new points to the solution set. This reachability check is performed in a robot simulator. If the pose of this point is not reachable by the robot, it will be ignored, and the next best point will be considered. 

The second point is more challenging and presents a fundamental challenge related to solving the inspection planning problem. Especially for robotic manipulators, it is expensive to perform path planning queries of long, complex paths in cluttered environments. This path planning requires checking the capability of a robot to execute consecutive linear motions. The sufficient number of such queries that are needed to keep the theoretical guarantees intact, would result in prohibitive computation times. This is mainly because the number of possible trajectories under consideration is exponential in the size of the input graph. Our final algorithm will neglect this check. Thus the near-optimality only remains intact if the following assumption is valid.
\begin{assumption}
If every $p\in P$ is reachable in at least one orientation $o\in O$ by the robot, and every view pose $x\in X\subseteq V$ is reachable (position + orientation), then we assume that the robot can move the sensor along the trajectory that minimizes $C(X)$ without incurring additional costs.
\label{ass:main}
\end{assumption}
This assumption is almost surely guaranteed in the case of drone inspections, if the resolution of the discretization is sufficiently fine. The view pose discretization that we proposed is aimed towards maximizing the probability that this assumption is also valid in other cases (e.g. manipulators). The grid structure of the view pose space graph ensures that the maximal spatial distance between reachability checks is limited to the diagonal distance of the input grid. The capability of our approach to deal with large graphs ensures that grids with a fine spacing can easily be achieved. Also, note that from a theoretical perspective this assumption only has to hold for the final solution. Partial solutions that violate Assumption \ref{ass:main} do not imply that the final solution will violate this assumption, or that $OPT$ will be violated. 

An improved implementation is however possible. This implementation would check all configurations in the complete path (i.e. the path that minimizes $C(X)$) in each iteration with a specific orientation. However, when such a path is not entirely reachable, complicated data structures are required to reflect this knowledge. This data structure is required because robotic manipulators can have multiple inverse kinematics solutions, which makes the reachability dependent on previous states. This would result in a substantial increase in the complexity of the implementation, and is therefore not considered in this work. This choice to omit this step is also motivated by our experiments where Assumption \ref{ass:main} was not violated.

\section{Experiments}
We will perform three different experiments in this section. The first experiment aims to evaluate the quality of solutions generated by the proposed algorithm in a variety of complex inspection tasks. The variety of problems arises from combining different robotic systems, different quality functions and different inspection objects. The second experiment evaluates the robustness of the GCB algorithm and post-processing step to changes in problem defining parameters. The third experiment aims to show that the proposed approach can solve highly complex, large scale inspection tasks. In this experiment, the focus is less on flexibility and more on the size of the problems.

To keep the variety of problems manageable we limit the number of quality functions to two specific functions. These functions are displayed in Fig.\ref{fig:covQual}. The first quality function models an inspection task where only coverage is important. Thus the measurement quality is not dependent on the measurement conditions. We, however, define a maximum measurement angle which is 30$^{\circ}$ in Fig.\ref{fig:covQual} (left) after which the quality becomes zero. Function $G$ (see Eq.\ref{eq:process}) is in both quality functions the $max$ operation. The second function is more complex and is given by:

\begin{equation}
    A = \frac{cos(\gamma)}{r^2}
\label{eqn::qual}
\end{equation}
Where $\gamma$ is the angle between the vector pointing from point $m$ to view pose $v$ and the surface normal and $r$ is the distance between these points. This function is proportional to the projected area of a surface element located at $m$ measured from $v$. It is important that that there is a minimum measurement distance that limits measurements that are too close. Note that any linear scaling of the distance leaves the problem unchanged. In experiment 1 we will consider both quality functions, and in experiment 2 we will only consider the second quality function.  

Both functions will result in fundamentally different submodular functions $f$. Function $C$ will result in a $f$ that can feature many nonlinearities since $q_i$ can only be zero or one. This will result in a more submodular $f$ (i.e. higher submodular curvature). $A$ will result in a smoother $Q$ since $q_i$ can take on all positive values. 
\begin{figure*}
    \centering
    \includegraphics[width=\textwidth]{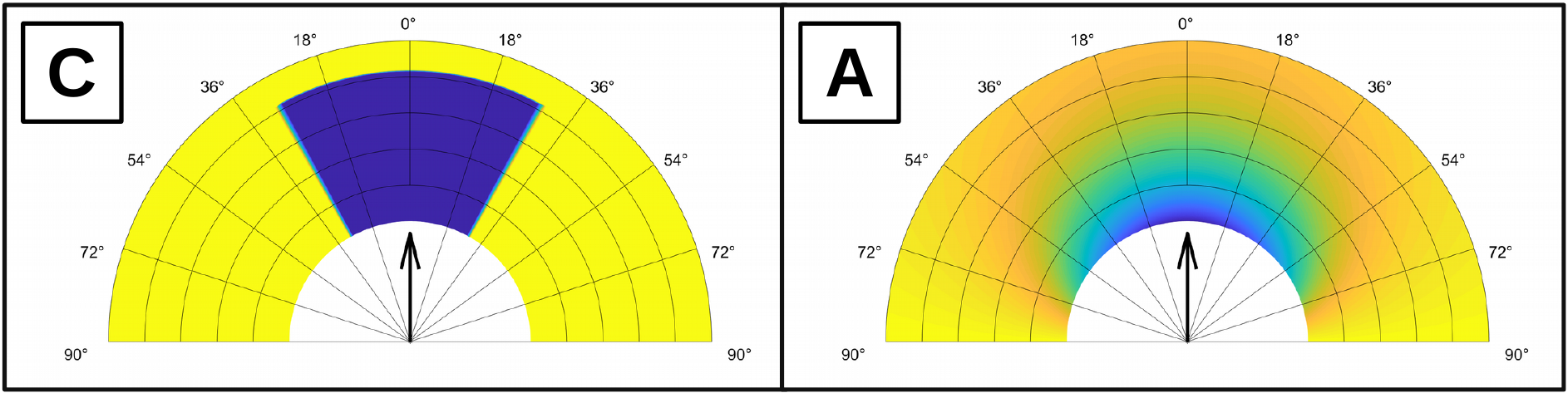}
    \caption{These plots show the quality functions that are used in the experiments. The black arrow in each figure represents the surface normal, and the colors display the quality value if a camera is located at that position pointing towards the surface element (located at the base of the black arrow). A high measurement quality is displayed as a blue color, and a low-quality value is shown as yellow. The left image is a standard coverage quality function, thus one if the measurement conditions are acceptable and zero otherwise. The right image shows a quality function proportional to the projected surface area of a surface element which is a smooth function.}
    \label{fig:covQual}
\end{figure*}

\subsection{Performance evaluation}
\begin{figure*}
    \centering
    \includegraphics[width=\textwidth]{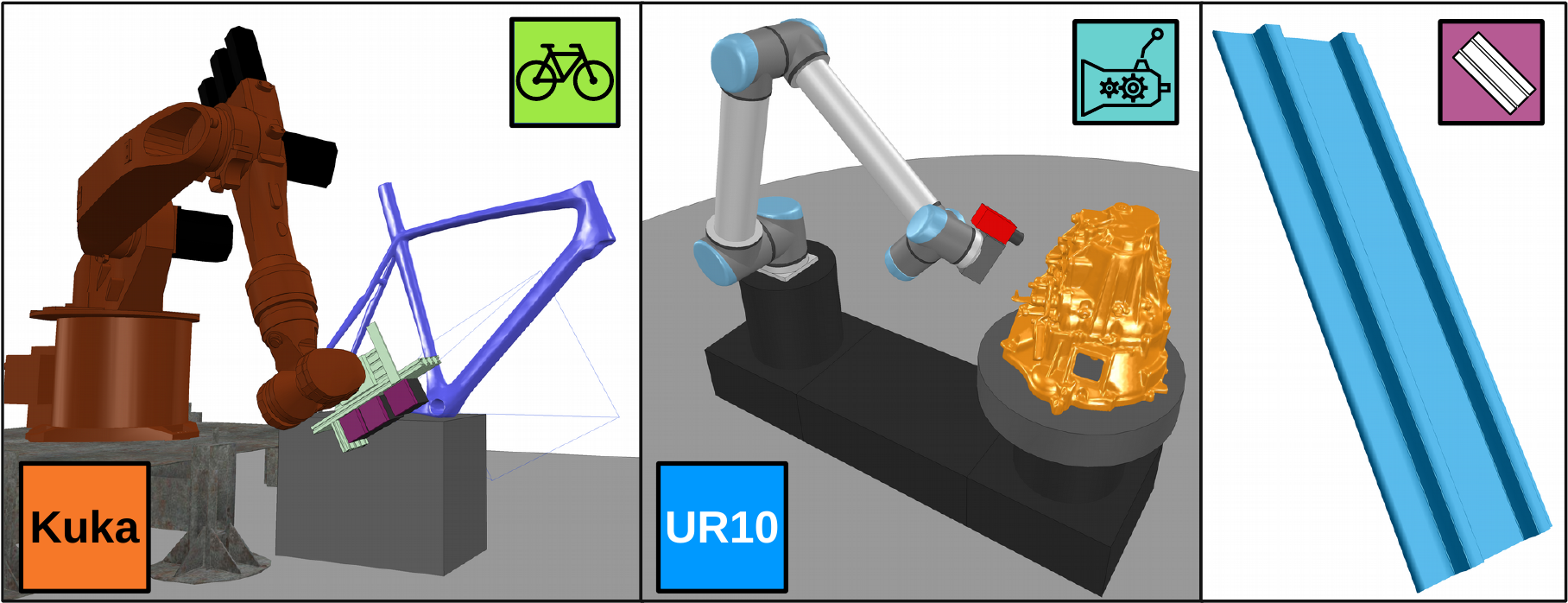}
    \caption{This figure shows the robot systems and objects that are considered in the experiment. The first measurement system consists of a Kuka KR16 robot with 6DoF, and a complicated measurement system. The second system consists of a smaller camera attached to a UR10 robot with 6DoF and an additional rotation table (7DoF in total). The first system represents a very constrained system while the second system is very flexible. The three measurement objects are a bicycle frame, a transmission housing and a plate with elevations.}
    \label{fig:simple}
\end{figure*}

During this experiment, we test the performance of the proposed algorithm in a wide range of problems. We create different problems by combining two different measurement systems, measuring three different objects with two different quality functions (creating 12 distinct problems). As a reference technique, we consider the traditional method that starts by building a set of viewpoints using the greedy algorithm, ignoring travelling costs, and proceeds by connecting them via the shortest path. This algorithm is also applied to the discretization proposed in this article. We will refer to this technique as the greedy algorithms which refers to the construction of the viewpoint set. We run the improvement algorithm of Sec.\ref{sec:improvement} on both the result of the GCB algorithm and the greedy algorithm. We will distinguish between both the original solution and the improved solution by adding a, +, sign. We will compare the performance of algorithms by percentages $Q(X)/OPT$ (Section \ref{sec:tight}).

The two measurement systems in this experiment represent a constrained system and an unconstrained system (see Fig.\ref{fig:simple}). The Kuka system features a large measurement device that limits the movements of this robot significantly due to complex self-collisions. This results in a reduced reachability of the manipulator. The UR10 system with rotation table has 7DoF in total, which makes it a less restricted system. Furthermore, the measurement device is smaller resulting in fewer self-collisions. The three objects that need to be inspected are fundamentally different. The bicycle frame is a large object with a tube structure. The transmission casing is a nearly convex object that still features complex occlusions. The plate object is an object that can be efficiently measured by a simple path that moves back an forth over the object.  

To limit the number of problems, we set all other problem-related parameters to be equal. However, we consider the effect of changing several important parameters in the next experiment. Parameter $\alpha$ in Sec.\ref{sec::travellingCost} is chosen to be 0.05. This means that performing a measurement costs as much as travelling 50mm with the end-effector. This choice reflects that we focus on continuous scanning, which requires that the robot at least slows down the speed of the end-effector surrounding any area of measurement. Distance function $d_t$ introduced in Eq.\ref{eq:dist} is the Euclidean distance between points in meters. Function $d_o$ is the angle in radians of the axis angle representation of the rotation between orientations. Parameter $\beta$ is chosen to be 0.01.

The spatial view position disctetization is obtained by dilating the convex hull of the object under consideration by 400mm. The resolution of the grid is 50mm in any dimension. Any unreachable point is rejected in advance in accordance with the technique proposed in Sec.\ref{sec::tspCost}. In this experiment we will mainly focus on the quality of the solutions returned by all algorithms. We will not focus on the execution time of each algorithm because our implementation is aimed towards flexibility rather than speed. Thus, all times that are provided are only suited to indicate execution times. The budget was chosen to 10 for all experiments but was increased until the total coverage of the trajectory was saturated. This ensures that the path is long enough to at least cover the entire object.

The perspective angle of the measurement device is 45$^{\circ}$ (arbitrary choice) in all experiments, except for the experiments where the plate object is being measured. The perspective angle in these experiments was changed to 15$^{\circ}$. This distinction was introduced to force to optimizer towards trajectories that pass multiple times over the plate object. While the maximum measurement angle in the coverage function is 30$^{\circ}$ in nearly all cases, it is 10$^{\circ}$ in case of the plate object. The results of this experiment are displayed in Fig.\ref{fig:table}. 

\begin{figure*}
    \centering
    \includegraphics[width=\textwidth]{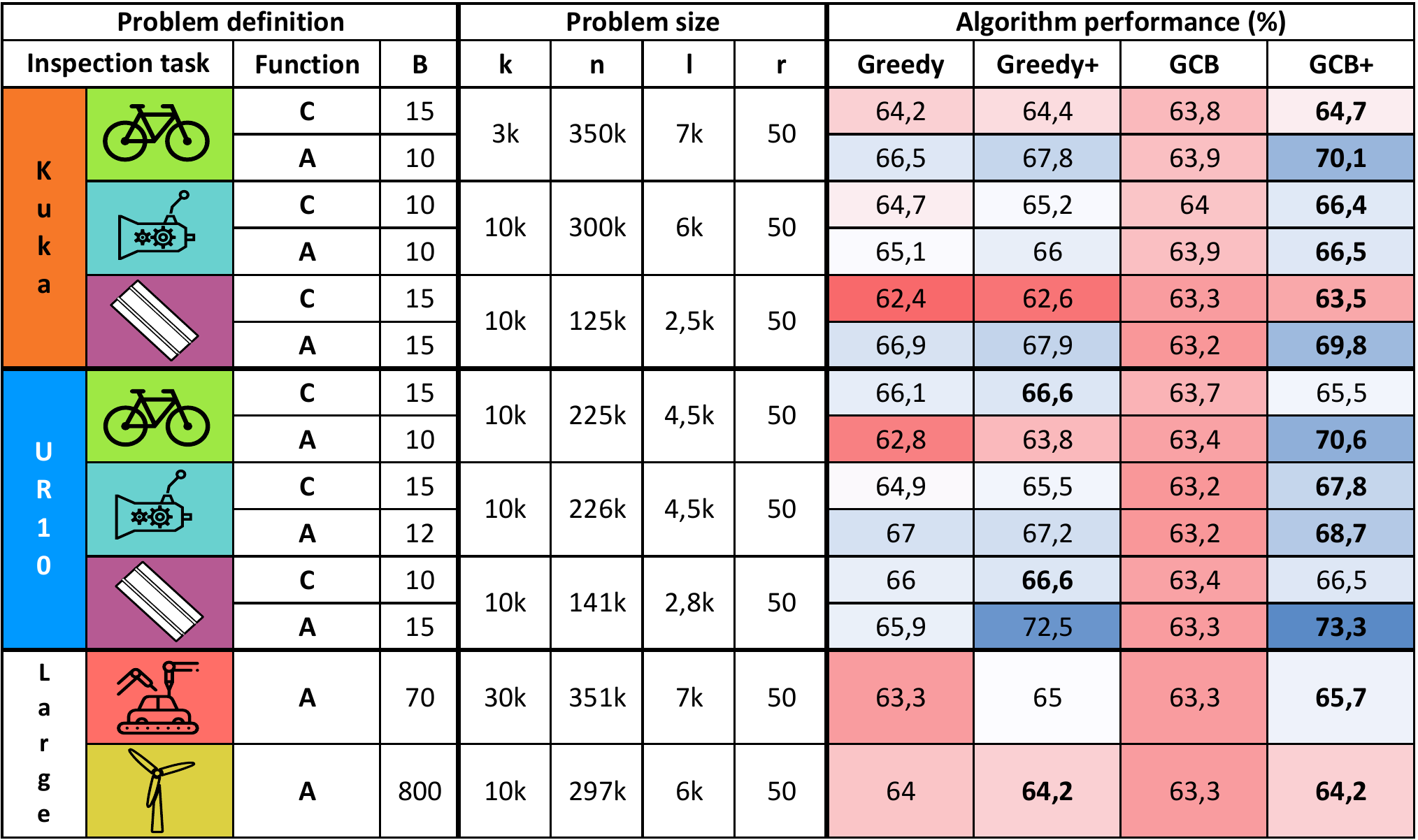}
    \caption{Figure summarizing the results that were obtained for the first and final experiment. The icons are introduced in Fig. \ref{fig:simple}, Fig.\ref{fig:covQual} and Fig. \ref{fig:tough}. The letters in problem size were introduced in Fig.\ref{fig:sheetSheet}. All the numbers under algorithm performance are percentages of $OPT$. Greedy is the result of running the traditional approach that separates the viewpoint collection (using the greedy algorithm) and path planning step. The $+$ refers to feeding the solution of the original approach to the improvement step introduced Sec. \ref{sec:improvement}.}
    \label{fig:table}
\end{figure*}
The maximum time required for precomputing the discretization in the smaller scale experiments was 10 minutes. The maximum time required to execute any algorithm was 1 hour. All these times are acceptable in real-world applications, especially since very large graphs were considered (i.e. 350k nodes) in this experiment. 

A first thing that is noticeable from Fig.\ref{fig:table} is that the traditional greedy algorithm performs surprisingly well. It almost always outperforms the basic GCB algorithm. This means that important viewpoints can almost always be connected with an efficient path. The range of different problems that were considered in these experiments indicates that this can safely be assumed in the design of new algorithms. The GCB+ algorithm almost always outperforms any other algorithm. However, the margin can be small. The degree to which GCB+ outperforms other methods is also dependent on the problem that is being considered. This observation also indicates that a proper evaluation of new algorithms requires the considerations of many different problems. Another observation is that the GCB+ algorithm performs better with the $A$ quality function. The smoothness of $A$ results in a smoother $f$ which is better suited to the local nature of the $^+$ step (Section \ref{sec:improvement}). 

\subsection{Robustness analysis}
In the previous experiment, we fixed some parameters that defined the inspection problem. In this experiment, to evaluate the stability of our algorithms, we will change these parameters to see the effect on the quality of the solutions generated by automated algorithms. In the previous experiment, we iteratively defined a travel budget for each experiment. To get an idea about the effect of this choice, we will change $B$, and keep all other parameters the same. In submodular orienteering the notion of cost is central. Therefore we will change both parameter $\beta$ (Eq.\ref{eq:dist}) and $\alpha$ (Eq.\ref{eq:alpha}) to see the effect of different types of cost functions. Parameter $\alpha$ determines the cost of a measurement, and $\beta$ determines whether orientation related costs, or position-related costs dominate. When $\alpha$ is changed, we also augment the budget. We augment the budget in such a way that the result of the greedy algorithm will remain the same. The main reason for this augmentation is that the GCB algorithm will always have a solution with more elements than the greedy algorithm. Changing $\alpha$ will indirectly penalize the number of elements in the final solutions. So changing $\alpha$ will test the ability of the GCB algorithm to deal with this changing context. 

We will run the algorithm with the different parameters on two different problems. These are both the problem that proved to be the most challenging, and the problem that proved to be the least challenging. Both problems incidentally considered the inspection of the plate object (see Fig.\ref{fig:simple}). This object was inspected by the Kuka kr16 robot, using the quality function $C$ in the most challenging problem. In the least challenging problem, this object was inspected by the UR10 robot with the quality function $A$. The results of these experiments are provided in Fig.\ref{fig:stability}.

\begin{figure*}
    \centering
    \includegraphics[width=\textwidth]{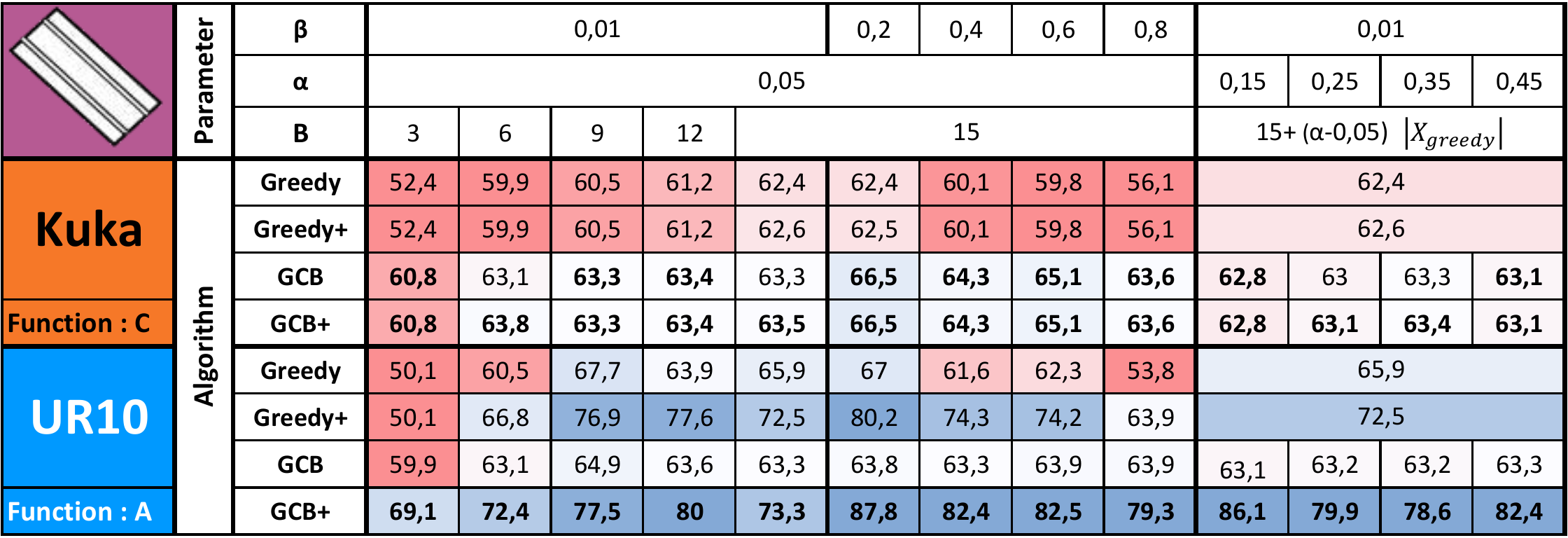}
    \caption{Figure showing the effect of changing problem defining parameters $\alpha$ (Eq.\ref{eq:alpha}), $\beta$ (Eq.\ref{eq:dist}) and the budget $B$ (Eq.\ref{eq:theproblem}). In this experiment only the best and the worst problems of previous experiment where considered.}
    \label{fig:stability}
\end{figure*}

It is clear that the performance of the GCB algorithm is robust to changing all investigated parameters. The only notable exception is when the budget is too small. Our augmentation step is however able to recover from this is the less challenging problem. The greedy algorithm is however far less robust to a change of parameters. 

\subsection{Large scale highly complex inspection tasks}
In this section, we will study the scalability of our approach by considering two very large complex yet distinct problems. The first problem considers the inspection of a car frame with a robotic manipulator on rails. This problem distinguishes itself by its very complex collision constraints due to complex collisions and self-collisions. The second problem considers the inspection of a 60-meter high wind turbine with a drone. This problem is different from the first problem since the trajectory of the drone is much longer. The quality of the solutions obtained in this experiment are included in Fig.\ref{fig:table}. 

The perspective angle of the measurement device in the car frame inspection problem is 40$^{\circ}$. The distances are computed precisely the same as in the first experiment. The convex hull of the car frame is dilated by 0.5m, and a grid with a resolution of 0.1m within this dilation is adopted to generate the viewpoint set. The car frame itself is a triangulation consisting of 373k triangles. To perform collisions, it is represented by an octree with 3525 voxels. Visibility is computed with the original mesh together with all other static scene meshes. Further details about the size of the problem are provided in Fig.\ref{fig:simple}. The drone in the wind turbine inspection problem is equipped with a camera with a perspective angle of 45$^{\circ}$. The wind turbine is provided as a triangulation of 103k triangles and is represented as a voxelization of 7k voxels to compute collisions. The convex hull of the wind turbine is dilated by 1m to create a viewpoint set with a resolution of 1m. 

\begin{figure*}
    \centering
    \includegraphics[width=\textwidth]{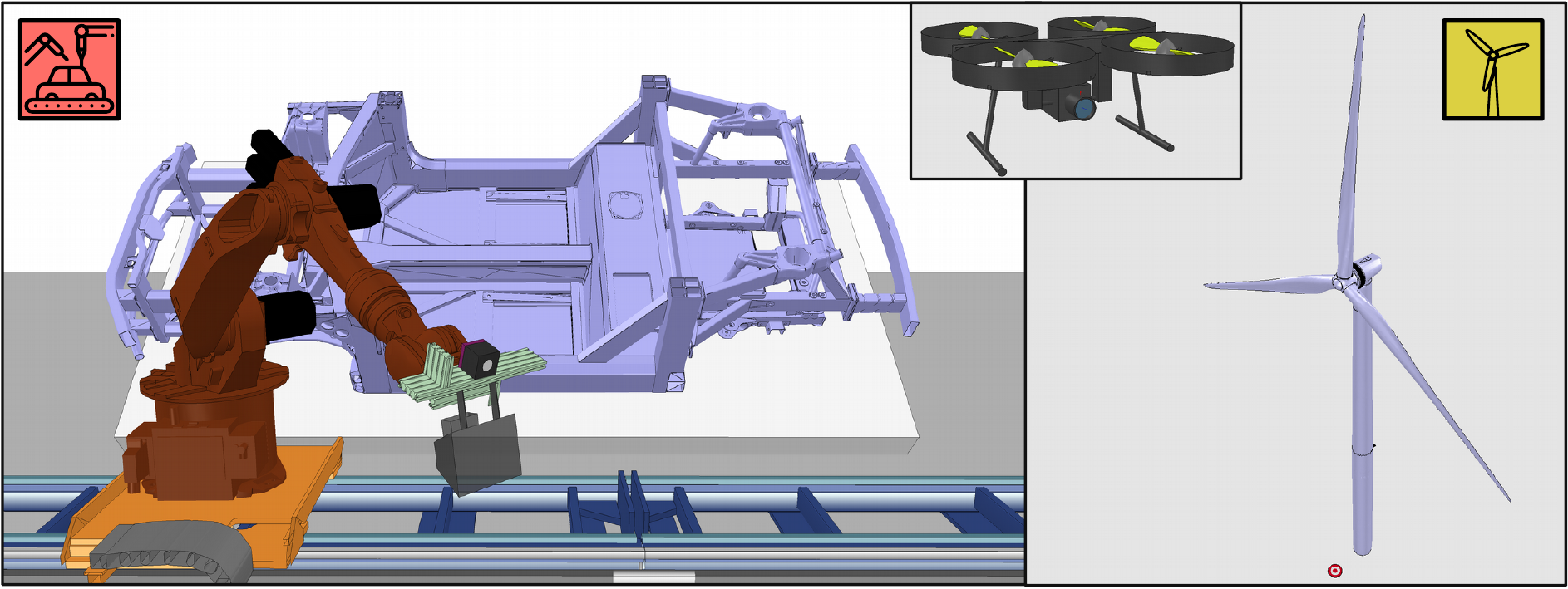}
    \caption{This figure shows the two complex inspection tasks considered in the experiment. The first inspection task entails the inspection of a car frame, with a complex measurement device. The movement of this device is provided by a Kuka KR16 robot placed on a linear translation stage. The second inspection task entails the inspection of a 60m high wind turbine with a drone. }
    \label{fig:tough}
\end{figure*}

The Greedy+ and GCB+ algorithms were allowed to run for 24 hours on both problems. The performance of the Greedy and GCB algorithms was extracted from running the + algorithms. Assumption \ref{ass:main} was found to be valid for both problems. For the turbine inspection problem, it was straightforward to find a path with linear connections between way-points. For the Car frame inspection problem it was more challenging because of the complexity of the robot system, self-collisions resulting from the measurement system and collisions from the complex car frame. To find this path we resorted to programming the path manually using a virtual reality programming approach\footnote{A video demonstrating this process is available online: \href{https://youtu.be/nakQGTs4Fs0}{https://youtu.be/nakQGTs4Fs0}}. The fact that we manually could find a path connecting all way-points using linear paths suffices to show the existence of such a path. 

The size of these problems is results from the size the input graph, and the length of the final path. The largest problem in terms of input graph is the car frame inspection problem, with an input graph of 351k nodes. Also notable is the size of the object discretization which is 30k points. The length of the path in this problem was 351 nodes of the input graph. The complexity of the car frame is also an important factor since interactive collision checks are more expensive. The path length of the wind turbine inspection problem is 520 nodes.
\section{Conclusion}

We started by linking the general robotic inspection planning problem to the submodular orienteering problem. While this connection was already noticed in specific examples before, the generality of this connection was never stressed. For the submodular orienteering problem, there exists an algorithm that solves this problem that provides formal mathematical guarantees. We extended this algorithm with a post-processing step to improve the solution even further. 

We investigated the assumptions about the real world problem that are required during the optimization phase and formalized this in an assumption. We proposed a discretization procedure for real-world problems such that this assumption is highly likely satisfied. This means that solutions to the abstract optimization problem (together with their guarantees) are valid in the practical context as well. The discretization procedure that we proposed is also able to deal with very large and complex problems as demonstrated in our experiments.

We subjected the proposed algorithms to a wide range of problems in which it almost always outperformed the traditional approach. We also showed that the algorithm was capable of solving highly complex inspection planning problems. The observation that our algorithm can provide solutions to complex inspection problems with mathematical guarantees, makes it an ideal reference method to benchmark more pragmatic algorithms that do not offer these guarantees. Furthermore, we showed that the advantage of having mathematical guarantees in the context of inspection planning resulted in a more robust algorithm.

\section*{Acknowledgements}
B.B was funded by Fonds Wetenschappelijk Onderzoek (FWO, Research-Foundation Flanders) under Doctoral (PhD) grant strategic basic research (SB) 1S26216N.

\bibliographystyle{agsm}
\bibliography{references}

\end{document}